\documentclass[runningheads]{llncs}

\usepackage{graphicx}
\usepackage{amsmath}
\usepackage{dsfont}
\usepackage[font=scriptsize]{subfig}
\usepackage{algorithm}
\usepackage{algorithmic}
\usepackage{todonotes}
\usepackage{paralist}
\usepackage{ulem}

\newcommand\ind[1]{\ensuremath{\mathds{1}\left[#1\right]}}

\def\argmin{\ensuremath{\mbox{argmin}}}
\def\argmax{\ensuremath{\mbox{argmax}}}

\newcommand{\A}{\mathcal{A}}
\newcommand{\R}{\mathcal{R}}

\newcommand{\X}{\mathcal{X}}
\newcommand{\Sd}{\mathcal{S}^d}

\newcommand{\Null}{\operatorname{Null}}

\begin{document}
\title{Data Poisoning Attacks in Contextual Bandits}
%
%
\author{Yuzhe Ma\inst{1} \and
Kwang-Sung Jun\inst{1} \and
Lihong Li\inst{2}\and
Xiaojin Zhu\inst{1}
}
\authorrunning{Y. Ma et al.}
%
\institute{University of Wisconsin-Madison\\
\email{ma234@wisc.edu, kjun@discovery.wisc.edu, jerryzhu@cs.wisc.edu}\and
Google Brain., Kirkland, WA, USA\\
\email{lihong@google.com}}
\maketitle              
\begin{abstract}
We study offline data poisoning attacks in contextual bandits, a class of reinforcement learning problems with important applications in online recommendation and adaptive medical treatment, among others.
We provide a general attack framework based on convex optimization and show that by slightly manipulating rewards in the data, an attacker can force 
the bandit algorithm to pull a target arm for a target contextual vector.
The target arm and target contextual vector are both chosen by the attacker.
That is, the attacker can hijack the behavior of a contextual bandit. 
We also investigate the feasibility and the side effects of such attacks, and identify future directions for defense. 
Experiments on both synthetic and real-world data demonstrate the efficiency of the attack algorithm.

\end{abstract}
\section{Introduction}

As an important step toward trustworthy AI, adversarial learning studies robustness of machine learning systems against malicious attacks~\cite{goodfellow2014explaining,joseph_nelson_rubinstein_tygar_2018}.
Training set poisoning is a type of attack where the adversary can manipulate the training data such that a machine learning algorithm trained on the poisoned data would produce a defective model. 
The defective model is often similar to a good model, but affords the adversary certain nefarious leverages~\cite{Alfeld2016Data,biggio2012poisoning,jagielski2018manipulating,li2016data,Mei2015Security,Mei2015Machine,zhao2018data}.
Understanding training set poisoning is essential to developing defense mechanisms.

Recent studies on training set poisoning attack focused heavily on supervised learning.
There has been little study on poisoning sequential decision making algorithms, even though they are widely employed in the real world.
In this paper, we aim to fill in the gap by studying training set poisoning against contextual bandits.
Contextual bandits are extensions of multi-armed bandits with side information and have seen wide applications in industry including news recommendation~\cite{li10contextual}, online advertising~\cite{chapelle14simple}, medical treatment allocation~\cite{kuleshov14algorithms}, and also promotion of users' well-being~\cite{greenewald17action}.

Let us take news recommendation as a running example for poisoning against contextual bandits.
A news website has $K$ articles (i.e., arms).
It runs an adaptive article recommendation algorithm (the contextual bandit algorithm) to learn a policy in the backend. 
Every time a user (represented by a context vector) visits the website, the website displays an article that it thinks is most likely to interest the user based on the historical record of all users. 
Then the website receives a unit reward if the user clicks through the displayed article, and receives no reward otherwise. 
Usually the website keeps serving users throughout the day and updates its article selection policy periodically (say, during the nights or every few hours). This provides an opportunity for an attacker to perform \textit{offline} data poisoning attacks, e.g. the attacker can sneak into the website backend at night before the policy is updated, and poison the rewards collected during the daytime.
The website unknowingly updates its policy with the poisoned data.  On the next day it behaves as the attacker wanted.

More generally, we study adversarial attacks in contextual bandit where the attacker poisons historical rewards in order to force the bandit to pull a target arm under a target context.
One can view this attack as a form of offline reward shaping~\cite{ng99policy}, but it is adversarial reward shaping.
Our main contribution is an optimization-based attack framework for this attack setting.
We also study the feasibility and side effect of the attack. We show on both synthetic and real-world data that the attack is effective. This exposes a security threat in AI systems that involve contextual bandits. 

\section{Review of Contextual Bandit}

This section reviews contextual bandits, which will be the victim of the attack in this paper. 
A contextual bandit is an abstraction of many real-world decision making problems such as product recommendation and online advertising.  Consider for example a news website which strives to recommend the most interesting news articles personalized for individual users.  Every time a user visits the website, the website observes certain contextual information that describes the user such as age, gender, location, past news consumption patterns, etc.  
The website also has a pool of candidate news articles, one of which will be recommended and shown to the user.  
If the recommended article is interesting, the user may click on it; otherwise, the user may click on other items on the page or navigate to another page.  
The click probability here depends on both the user (via the context) and the recommended article.
Such a dependency can be learned based on click logs and used for better recommendation for future users.

An important aspect of the problem is that the click feedback is observed only for the recommended article, 
not for others. 
In other words, the decision (choosing which article to show to a user) is irrevocable; it is impractical to force the user to revisit the webpage so as to recommend a different article.  As a result, the feedback data being collected is necessarily biased towards the current recommendation algorithm being employed by the website, raising the need for balancing \textit{exploration} and \textit{exploitation} when choosing arms~\cite{li10contextual}.  This is in stark contrast to a typical prediction task solved by supervised learning where predictions do not affect the data collection.


Formally, a contextual bandit has a set $\X$ of contexts and a set $\A=\{1,2,\ldots,K\}$ of $K$ arms.   A contextual bandit algorithm proceeds in rounds $t=1,2,\ldots$.  At round $t$, the algorithm observes a context vector
$x_t\in \R^d$, chooses to pull an arm $a_t\in\A$, and observes a reward $r_t \in \R$.  The goal of the algorithm is to maximize the total reward garnered over rounds.  In the news recommendation example above, it is natural to define $r_t=1$ if user clicks on the article and $0$ otherwise, so that maximizing clicks is equivalent to maximizing the click-through rate, a critical business metric in online recommender systems.

In this work, we focus on the most popular and well-studied setting called linear bandits, where the expected reward is linear map of the context vector.  Specifically, we assume each arm $a$ is associated with an unknown vector $\theta_a\in\R^d$ with $\Vert \theta_a \Vert_2\le S$, so that for every $t$:
\begin{equation}
\label{eq:reward}
r_t=x_t^\top \theta_{a_t}+\eta_t\,,
\end{equation}
where $\eta_t$ is a $\sigma$-subGaussian noise. For simplicity, we assume $\eta_t$ is unbounded and thus the reward can take any value in $\R$.

Most contextual bandit algorithms adopt the optimism-in-face-of-uncertainty (OFU) principle for efficient exploration.
The OFU principle constructs an Upper Confidence Bound (UCB) for the mean reward of each arm based on historical data and then selects the arm with the highest UCB at each time step~\cite{auer02finite,abbasi11improved}. 
In round $t$, the historical data consists of the context, action, reward triples $(x,a,r)$ from the previous $t-1$ rounds.
It is useful to split the historical data so that the feedback from the same arm is pooled together. 
Define $[K] = \{1,\ldots,K\}$.
Let $m_a$ be the number of times arm $a$ was pulled up to time $t-1$.
This implies that $\sum_{a \in [K]} m_a = t-1$.
For each $a \in [K]$, let $X_a\in \R^{m_a\times d}$ be the design matrix for rounds, where arm $a$ was pulled and each row of $X_a$ is a previous context.  Similarly, let $y_a\in\R^{m_a}$ be the corresponding reward (column) vector.

A UCB-style algorithm first forms a point estimate of $\theta_a$ by ridge regression
\begin{equation}
\label{eq:hattheta}
\hat\theta_a=(X_a^\top X_a+\lambda I)^{-1}X_a^\top y_a,~~~~ \forall a\in[K],
\end{equation}
where $\lambda > 0$ is a regularization parameter.
At round $t$, the algorithm observes the context $x_t$ and then selects the arm with the highest UCB:
\begin{equation}\label{armselect}
a_t=\argmax_{a\in[K]} \left\{ x_t^\top\hat\theta_a+\alpha_a\Vert x_t\Vert_{V_a^{-1}}\right\}\,,
\end{equation}
where 
$\Vert x_t\Vert_{V_a^{-1}}=\sqrt{x_t^\top V_a^{-1}x_t}$ is the Mahalanobis norm
and $V_a=X_a^\top X_a+\lambda I$. 
Intuitively, for less frequently chosen $a$, the second term above tends to be large, thus encouraging exploration.
The exploration parameter $\alpha_a$ is algorithm-specific. 
For example, in LinUCB~\cite{li10contextual} $\alpha_a=1+\sqrt{\frac{1}{2}\log\frac{2}{\delta}}$ and in OFUL~\cite{abbasi11improved}
$\alpha_a=\sigma\sqrt{2\log(\frac{\text{det}(V_a)^\frac{1}{2}\text{det}(\lambda  I)^{-\frac{1}{2}}}{\delta})}+\lambda^\frac{1}{2}S$, where $\delta>0$ is a confidence parameter. 
Here, we assume $\alpha_a$ may depend on input parameters like $\delta$ and observed data up to $t-1$, but not $x_t$.

In Algorithm~\ref{alg:learner}, we summarize the contextual bandit algorithm.
While the bandit algorithm updates its $\hat\theta$ estimates in every round (step 3), in practice due to various considerations such updates often happen in mini-batches, 
e.g., several times an hour, or during the nights when fewer users visit the website~\cite{li10contextual,agarwal16multiworld}. 
Between these consecutive updates, the bandit algorithm follows a fixed policy obtained from the last update.

\begin{algorithm}
  \begin{algorithmic}[1]
    \caption{Contextual bandit algorithm}
    \label{alg:learner}
    \STATE \textbf{Parameters}: confidence $\delta$, regularizer $\lambda$, UCB function $\alpha$.
    \FOR{$t=1,2,\ldots, T$}
    \STATE Receive context $x_t$, estimate $\hat\theta_a, a\in [K]$ with~\eqref{eq:hattheta}.
    \STATE Pull arm $a_t=\argmax_{a\in[K]} \left\{x_t^\top\hat\theta_a+\alpha_a\Vert x_t\Vert_{V_a^{-1}}\right\}$.
    \STATE World generates reward $r_t=x_t^\top \theta_{a_t}+\eta_t$.
    \STATE Append $x_t$ and $ r_t$ to  $X_{a_t}$ and $y_{a_t}$, respectively.
    \ENDFOR
  \end{algorithmic}
\end{algorithm}

%
%

\section{Attack Algorithm in Contextual Bandit}

We now introduce an attacker with the following attack goal:
\begin{quote}
\textbf{Attack goal $\mathbf{[x^*\rightarrow a^*]}$}: On a particular attack target context $x^*$, force the bandit algorithm to pull an attack target arm $a^*$. 
\end{quote}
For example, the attacker may want to manipulate the news service so that a particular article $a^*$ is shown to users $x^*$ from certain political bases.
The attack is aimed at the current round $t$, or more generally the whole period when the arm-selection policy is fixed.
Any suboptimal arm $a^*$ can be the target arm. 
For concreteness, in our experiments the attacker always picks the worst arm $a^*$ as the target arm.
This is defined in the sense of the worst UCB, namely replacing $\argmax$ with $\argmin$ in~\eqref{armselect}, resulting in the target arm in~\eqref{eq:worstUCB}. 

We assume the attacker has full knowledge of the bandit algorithm and has access to all historical data.
The attacker has the power to poison the historical reward vector\footnote{In this paper we restrict the poisoning to modifying rewards for ease of exposition.  More generally, the attacker can add, remove, or modify both the rewards and the context vectors.  Our optimization-based attack framework can be generalized to such stronger attacks, though the optimization could become combinatorial.}  $y_a$, $\forall a\in [K]$. 
Specifically, the attacker can make arbitrary modifications $\Delta_a\in\R^{m_a}$, $\forall a\in [K]$ so that the reward vector for arm $a$ becomes $y_a+\Delta_a$. 
After the poisoning attack, the ridge regression performed by the bandit algorithm yields a different solution:
\begin{equation}
\hat\theta_a=V_a^{-1}X_a^\top(y_a+\Delta_a).
\end{equation}
Because such attacks happen on historical rewards in between bandit algorithm updates, we call it offline.

Now we can formally define the attack goal.
\begin{definition}[Weak attack] A target context $x^*$ is called \emph{weakly attacked} into pulling target arm $a^*$ if after attack the following inequalities are satisfied:
\begin{equation}\label{attack:weak}
{x^*}^\top\hat\theta_{a^*}+\alpha_{a^*}\Vert x^* \Vert_{V_{a^*}^{-1}}>{x^*}^\top\hat\theta_a+\alpha_a\Vert x^* \Vert_{V_a^{-1}}, ~~\forall a\neq a^*.
\end{equation}
In other words, the algorithm is manipulated into choosing $a^*$ for context $x^*$.
\end{definition}

To avoid being detected, the attacker hopes to make the poisoning $\Delta_a, a\in[K]$ as small as possible.
We measure the magnitude of the attack by the squared $\ell_2$-norm $\sum_{a\in[K]}\Vert\Delta_a\Vert_2^2$.\footnote{The choice of norm is application dependent, see e.g.,~\cite[Figure 3]{Mei2015Machine}.  Any  norm works for the attack formulation.} 
We therefore formulate the attack as the following optimization problem:
\begin{equation}\label{attack:optimization_weak}
\begin{aligned}
\min_{\Delta_a: a\in [K]}\quad &\sum_{a\in [K]}\Vert\Delta_a\Vert_2^2\\
\text{s.t. \quad} &{x^*}^\top\hat\theta_{a^*}+\alpha_{a^*}\Vert x^* \Vert_{V_{a^*}^{-1}}
	> {x^*}^\top\hat\theta_a+\alpha_a\Vert x^* \Vert_{V_a^{-1}}, \forall a\neq a^*\\
\text{where \quad}&\hat\theta_a=V_a^{-1}X_a^\top(y_a+\Delta_a), ~~\forall a.
\end{aligned}
\end{equation}

The weak attack above ensures that, given the target context $x^*$, the bandit algorithm is forced to pull arm $a^*$ instead of any other arms.
Unfortunately, the constraints do not result in a closed convex set.
To formulate the attack as a convex optimization problem, we introduce a stronger notion of attack that implies weak attack:
\begin{definition}[Strong attack] A target context $x^*$ is called \emph{$\epsilon$-strongly attacked} into pulling target arm $a^*$, for some $\epsilon>0$, if after attack the following holds:
\begin{equation}\label{attack:strong}
{x^*}^\top\hat\theta_{a^*}+\alpha_{a^*}\Vert x^* \Vert_{V_{a^*}^{-1}}\ge \epsilon+{x^*}^\top\hat\theta_a+\alpha_a\Vert x^* \Vert_{V_a^{-1}}, ~~\forall a\neq a^*\,.
\end{equation}
\end{definition}
This is essentially a large margin condition which requires the UCB of $a^*$ to be at least $\epsilon$ greater than the UCB of any other arm $a$.
The margin parameter $\epsilon$ is chosen by the attacker.
We achieve strong attack with the following optimization problem:
\begin{equation}\label{attack:optimization_strong}
\begin{aligned}
\min_{\Delta_a: a\in [K]}\quad &\sum_{a\in [K]}\Vert\Delta_a\Vert_2^2\\
\text{s.t. \quad} &{x^*}^\top\hat\theta_{a^*}+\alpha_{a^*}\Vert x^* \Vert_{V_{a^*}^{-1}}
	\ge\epsilon+{x^*}^\top\hat\theta_a+\alpha_a\Vert x^* \Vert_{V_a^{-1}}, ~~\forall a\neq a^*\\
\text{where \quad}&\hat\theta_a=V_a^{-1}X_a^\top(y_a+\Delta_a), \forall a.
\end{aligned}
\end{equation}
The optimization problem above is a quadratic program 
with linear constraints in $\{\Delta_a\}_{a\in[K]}$.
We summarize the attack in Algorithm~\ref{alg:attacker}.
In the next section we discuss when the algorithm is feasible.

\begin{algorithm}
  \begin{algorithmic}[1]
    \caption{Data Poisoning Attack in Contextual Bandit}
    \label{alg:attacker}
    \STATE \textbf{Input}: victim contextual bandit (Algorithm~\ref{alg:learner}), target context $x^*$, target arm $a^*$, attack margin $\epsilon$, historical data $X_a, y_a, a\in [K]$.
    \STATE Solve~\eqref{attack:optimization_strong} for $\Delta_a, \forall a\in [K]$.
    \STATE If a solution $\Delta_a$ is found, poison $y_a\leftarrow y_a+\Delta_a$; otherwise return \texttt{infeasible}.
  \end{algorithmic}
\end{algorithm}

\section{Feasibility of Attack}
While one can always write down the training set attack algorithm as optimization~\eqref{attack:optimization_strong}, there is no guarantee that such attack is feasible.   
In particular, the inequality constraints may result in an empty set.
One may naturally ask: are there context vectors $x^*$ that simply cannot be strongly attacked?\footnote{Even if some context $x^*$ cannot be strongly attacked, the attacker might be able to weakly attack it.
Weak attack is sufficient for the attacker to force an arm pull of $a^*$. 
However, as $\epsilon\rightarrow 0$ strong attack approaches weak attack.  Thus we only need to characterize strong attacks.}
In this section we present a full characterization of the feasibility question for strong attack.
As we will see, attack feasibility depends on the original training data.
Understanding the answer helps us to gauge the difficulty of poisoning, and may aid the design of defenses.

The main result of this section is the following theorem that characterizes a sufficient and necessary condition for the strong attack to be feasible.
\begin{theorem}\label{thm:infeasible}
A context $x$ cannot be strongly attacked into pulling $a^*$ if and only if there exists $a\neq a^*$ such that the following two conditions are both satisfied: 
\begin{compactdesc}
\item{(i) $x\in\Null(X_{a^*})\cap\Null(X_a)$, and}
\item{(ii) $\alpha_{a^*}||{x}||_{ V^{-1}_{a^*}}<\epsilon+\alpha_a||{x}||_{ V^{-1}_a}$.}
\end{compactdesc}
\end{theorem}
Before presenting the proof, we first provide intuition.
The key idea is that a context $x$ cannot be strongly attacked if some non-target arm $a$ is always better than $a^*$ for $x$ for any attack.
This can happen because there are two terms in the arm selection criterion~\eqref{armselect} while the attack can affect the first term only. 
It turns out that under the condition $(i)$ the first term becomes zero. 
If there exists a non-target arm that has a larger second term than that of the target arm (the condition $(ii)$), then no attack can force the bandit algorithm to choose the target arm.


We present an empirical study on the feasibility of attack in Section~\ref{subsec:feasibility}.

\begin{lemma}\label{lem:algebraiclem1}
$ x\in\Null(X_{a^*}) \Leftrightarrow {x}^\top  V^{-1}_{a^*} X_{a^*}^\top=  0$, where $ V^{-1}_{a^*}= X_{a^*}^\top X_{a^*}+\lambda I$.
\end{lemma}
\begin{proof}
First, we prove $x\in\Null(X_{a^*})\Rightarrow {x}^\top  V^{-1}_{a^*} X_{a^*}^\top=  0$. Note that 
\begin{equation}
\begin{aligned}
 x\in\Null(X_{a^*})&\Rightarrow X_{a^*}x=  0\\
 &\Rightarrow  X_{a^*}^\top X_{a^*}x=  0\\
&\Rightarrow( X_{a^*}^\top X_{a^*}+\lambda I)x= \lambda x\\
&\Rightarrow \frac{1}{\lambda}x=( X_{a^*}^\top X_{a^*}+\lambda I)^{-1}x= V^{-1}_{a^*}x.
\end{aligned}
\end{equation}
Therefore, we have
\begin{equation}
{x}^\top  V^{-1}_{a^*} X_{a^*}^\top= \frac{1}{\lambda}{x}^\top X_{a^*}^\top = \frac{1}{\lambda}( X_{a^*}x)^\top= 0.
\end{equation}
Now we show the other direction. Note that 
\begin{equation}
\begin{aligned}
{x}^\top V^{-1}_{a^*} X_{a^*}^\top=  0&\Rightarrow {x}^\top V^{-1}_{a^*} X_{a^*}^\top X_{a^*}=0\\
&\Rightarrow{x}^\top V^{-1}_{a^*}(V_{a^*} -\lambda I)=0\\
&\Rightarrow{x}^\top=\lambda{x}^\top V^{-1}_{a^*}\\
&\Rightarrow( X_{a^*}^\top X_{a^*}+\lambda I)x=\lambda x\\
&\Rightarrow X_{a^*}^\top X_{a^*}x=0\\
&\Rightarrow x^\top X_{a^*}^\top X_{a^*}x=0\\
&\Rightarrow \Vert X_{a^*}x\Vert_2^2 =0\\
&\Rightarrow X_{a^*}x=0\,,
\end{aligned}
\end{equation}
which implies $x\in\Null(X_{a^*})$.\qed

\end{proof}

\begin{proof}[Theorem~\ref{thm:infeasible}]
($\Leftarrow$)
According to lemma~\ref{lem:algebraiclem1}, condition $(i)$ implies 
\begin{equation}
{x}^\top  V^{-1}_{a^*} X_{a^*}^\top( y_{a^*}+\Delta_{a^*})={x}^\top  V^{-1}_a X_{a}^\top( y_{a}+\Delta_{a})=0.
\end{equation}
Combined with (ii) we have for any $\Delta_{a^*}$ and $\Delta_a$,
\begin{equation}
\begin{aligned}
&{x}^\top  V^{-1}_{a^*} X_{a^*}^\top( y_{a^*}+\Delta_{a^*})+\alpha_{a^*}||{x}||_{ V^{-1}_{a^*}} ~~=~~ \alpha_{a^*}||{x}||_{ V^{-1}_{a^*}}\\
&<~~ \epsilon+\alpha_a||{x}||_{ V^{-1}_a} ~~=~~ \epsilon+\alpha_a||{x}||_{ V^{-1}_a}+{x}^\top  V^{-1}_a X_{a}^\top( y_{a}+\Delta_{a})\,.
\end{aligned}
\end{equation}
Thus, ${x}$ cannot be attacked.

($\Rightarrow$)
This is equivalent to prove if $\forall a\neq a^*, \neg(i)\vee \neg(ii)$, then $x$ can be attacked. To show $x$ can be attacked, it suffices to find a solution for the optimization problem. 

If $\neg (i)$, then $ X_{a^*}x\neq  0$ or $ X_{a}x\neq  0$. Assume $ X_{a^*}x\neq  0$ (similar for the case $ X_{a}x\neq 0$), then ${x}^\top  V^{-1}_{a^*} X_{a^*}^\top\neq  0$. Let $p= X_{a^*} V^{-1}_{a^*}x$. For any $a\neq a^*$, arbitrarily fix some $\Delta_a$, then define 
\begin{equation}
q_a=\epsilon+\alpha_a||{x}||_{ V^{-1}_a}+{x}^\top  V^{-1}_a X_{a}^\top( y_{a}+\Delta_{a})-{x}^\top  V^{-1}_{a^*} X_{a^*}^\top  y_{a^*}-\alpha_{a^*}||{x}||_{ V^{-1}_{a^*}}.
\end{equation}
 Let $\Delta_{a^*}=kp$, where $k=\max_{a\neq a^*}\frac{q_a}{\Vert p\Vert_2^2}$. Thus,
 \begin{equation}
{x}^\top  V^{-1}_{a^*} X_{a^*}^\top\Delta_{a^*}=p^\top \Delta_{a^*}=k\Vert p\Vert_2^2\ge \frac{q_a}{\Vert p\Vert_2^2}\Vert p\Vert_2^2=q_a, ~~\forall a\neq a^*.
 \end{equation}
Therefore, we have for all $a \ne a^*$ that
\begin{equation}
{x}^\top  V^{-1}_{a^*} X_{a^*}^\top( y_{a^*}+\Delta_{a^*})+\alpha_{a^*}||{x}||_{ V^{-1}_{a^*}}\ge\epsilon+\alpha_a||{x}||_{ V^{-1}_a}+{x}^\top  V^{-1}_a X_{a}^\top( y_{a}+\Delta_{a})\,,
\end{equation}
which means $x^*$ can be attacked. 

If $\neg (ii)$, simply letting $\Delta_{a^*}=- y_{a^*}$ and   $\Delta_{a}=- y_{a}$ suffices, concluding the proof.\qed
\end{proof}

\section{Side Effects of Attack}

While the previous section characterized contexts $x^*$ that cannot be strongly attacked, this section asks an opposite question: suppose the attacker was able to strongly attack some $x^*$ by solving~\eqref{attack:optimization_strong}, what other contexts $x$ are affected by the attack? For example, there might exist some context $x \neq x^*$ whose pre-attack chosen arm is $a(x)=1$, but becomes $a'(x)=2$.
The side effects can be construed in two ways: on one hand the attack automatically influence more contexts than just $x^*$; on the other hand they make it harder for the attacker to conceal an attack.  
The latter may be utilized to facilitate detection by a defender. In this section, we study the side effect of attack and provide insights into future research directions on defense.

The side effect is quantified by the fraction of contexts in the context space such that the chosen arm is changed by the attacker. Specifically, let $\X$ be the context space and $P$ be a probability measure over $\X$. 
Let $a(x)$ and $a'(x)$ be the pre-attack and post-attack chosen arm of a context $x$. Then the \textit{side effect fraction} is defined as:
\begin{equation}\label{SER}
s = \int_{x\in\X} \ind{a(x)\neq a'(x)}P(x)dx\,.
\end{equation}
One can compute an \textit{empirical side effect fraction} $\hat s$ as follows. First sample $m$ contexts from $P$, and then let $\hat s=\frac{1}{m}\sum_{i=1}^m \ind{a(x)\neq a'(x)}$.  It is easy to show using Chernoff bound that $|s-\hat s|$ decays to $0$ at the rate of $1/\sqrt{m}$.

We now give some properties of the side effect. Specifically, we first show if $x$ is affected by the attack, $cx$ is also affected by the attack for any $c>0$.
\begin{proposition}\label{thm:side_effect}
If a context $x$ satisfies $a(x)\neq a'(x)$, then $a(cx)\neq a'(cx)$ for any $c>0$, where $a(x)$ and $a'(x)$ are the pre-attack and post-attack chosen arm of $x$. Moreover, $a'(cx)=a'(x)$, i.e., the post-attack chosen arms for $cx$ and $x$ are exactly the same.
\end{proposition}
\begin{proof}
First, for any $a \ne a'(x)$, define
\begin{equation}
f_a(x)= x^\top\hat\theta_{a'(x)}+\alpha_{a'(x)}\Vert x \Vert_{V_{a'(x)}^{-1}}- {x}^\top\hat\theta_a-\alpha_a\Vert x \Vert_{V_a^{-1}} \,.
\end{equation}
Note that $a'(x)$ is the best arm after attack, thus $f_a(x)>0$, $\forall a\neq a'(x)$. Therefore, for any $c>0$, we have
\begin{equation}
f_a(cx) =c f_a(x)>0, ~~\forall a\neq a'(x)\,,
\end{equation}
which implies that $a'(cx) = a'(x)$. The same argument may be used to show $a(cx)=a(x)$. Therefore, $a'(cx) = a'(x) \neq a(x) = a(cx)$.
\end{proof}
Proposition~\ref{thm:side_effect} shows that if a context $x$ has a side effect, all contexts on the open ray $\{cx: c>0\}$ also have the same side effect.

\begin{proposition}\label{thm:side_effect_strong_attacked}
If a context $x$ is strongly attacked, then $cx$ is also strongly attacked for any $c\ge1$.
\end{proposition}
\begin{proof}
First, for any $a \ne a^*$, define
\begin{equation}
f_a(x)= x^\top\hat\theta_{a^*}+\alpha_{a^*}\Vert x \Vert_{V_{a^*}^{-1}}- {x}^\top\hat\theta_a-\alpha_a\Vert x \Vert_{V_a^{-1}}\,.
\end{equation}
Since $x$ is strongly attacked, we have $f_a(x)\ge\epsilon$, $\forall a\neq a^*$. Therefore $f_a(cx)=cf_a(x)\ge f_a(x)\ge \epsilon$, which shows that $cx$ is also strongly attacked.
\end{proof}
%
The above propositions are weak in that they do not directly quantify the side effect fraction $s$.
They only tell us that when there is side effect, the affected contexts form a collection of rays.
In the experiment section we empirically study the side effect fraction.  Further theoretical understanding of the side effect is left as a future work.

\section{Experiments}

Our proposed attack algorithm works for any contextual bandit algorithm taking the form~\eqref{armselect}.  Throughout the experiments, we choose to attack the OFUL algorithm that has a tight regret bound and can be efficiently implemented.

\subsection{Attack Effectiveness and Effort: Toy Experiment}

To study the effectiveness of the attack, we consider the following toy experiment. The bandit has $K=5$ arms, and each arm has a payoff parameter $\theta_a\in \R^d$ where $d=10$, distributed uniformly on the $d$-dimensional sphere, denoted $\Sd$. 
To generate $\theta_a$, we first draw from a $d$-dimensional standard Gaussian distribution, $\tilde\theta_a\sim \mathcal{N}(\textbf{0}, \textbf{$I_d$})$ and then normalize: $\theta_a=\tilde{\theta_a}/\Vert\tilde \theta_a\Vert_2$. 

Next, we construct the historical data as follows. 
We generate $n=10^3$ historical context vectors $\{x_1, \ldots, x_n\}$ again uniformly on 
$\Sd$.
For each historical context $x$,
we pretend the world generates all $K$ rewards $\{r_a : a\in\A\}$ from the $K$ arms according to~\eqref{eq:reward}, where we set the noise level to $\sigma=0.1$.
We then choose an arm $a$
randomly from 
a multinomial distribution: $a \sim \mathrm{multi}(p_1,p_2,...,p_K)$, where $p_{i'}=\frac{\exp(r_{i'})}{\sum_{i'\in\A}\exp(r_{i'})}$. 
This forms one data point $(x, a, r_{a})$, and we repeat it for all $n$ points. 
We then group the historical data to form the appropriate matrices $X_a, y_a$ for every $a \in \A$.
Note that the historical data generated in this way is off-policy with respect to the bandit algorithm.
The regularization and confidence parameters are $\lambda=1$ and $\delta=0.05$, respectively.

In each attack trial, we draw a single target context $x^* \in \R^d$ uniformly from 
$\Sd$. 
Without attack, the bandit would have chosen the arm with the highest UCB based on historical data~\eqref{armselect}.
To illustrate the attack, we will do the opposite and set the attack target arm $a^*$ as the one with the smallest UCB instead:
\begin{equation}
\label{eq:worstUCB}
a^*=\argmin_{a\in[K]} \left\{{x^*}^\top \hat\theta_a+\alpha_a\Vert x^*\Vert_{V_a^{-1}}\right\},
\end{equation}
where $\alpha_a$ is the UCB parameter of the OFUL algorithm~\cite{abbasi11improved}. We set the strong attack margin as $\epsilon=0.001$. 
We then run the attack on $x^*$ with Algorithm~\ref{alg:attacker}.

We run $100$ attack trials. In each trial the arm parameters, historical data, and the target context $x^*$ are regenerated. 
We make two main observations:
\begin{compactenum}
\item The attacker is effective.  All $\epsilon$-strongly attacks are successful.  
\item The attacker's poisoning $\Delta$ is small.  
The total poisoning can be measured by $\Vert \Delta\Vert_2 = \sqrt{\sum_{a\in[K]}\Vert \Delta_a\Vert_2^2}$ in each attack trial. 
However, this quantity depends on the scale of the original pre-attack rewards $y_a$.
It is more convenient to look at the \emph{poisoning effort ratio}:
\begin{equation}\label{def:poisoning_effect_ratio}
\frac{\|\Delta\|_2}{ \|y\|_2 } = \sqrt{\sum_{a\in[K]}\Vert \Delta_a\Vert_2^2 \over \sum_{a\in[K]}\Vert y_a\Vert_2^2}.
\end{equation}
Figure~\ref{Attack:cost} shows the histogram for the poisoning effort ratio of the $100$ attack trials. 
The ratio tends to be small, with a median of $0.26$, which demonstrates that the attacker needs to only manipulate about $26\%$ of the rewards.
\end{compactenum}
These two observations indicate that poisoning attack in contextual bandit is easy to carry out.

\begin{figure}[H]
\centering
\includegraphics[width=0.55\textwidth,height=0.4\textwidth]{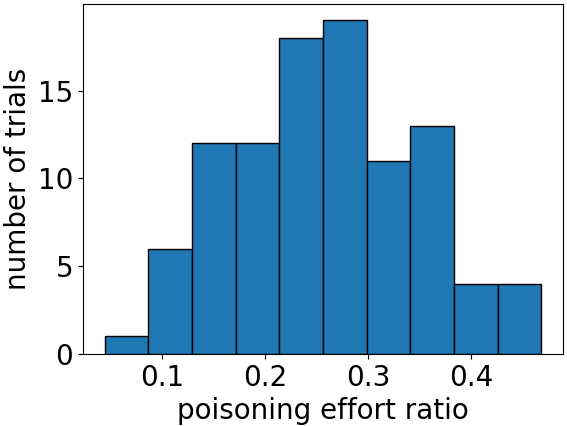}
\caption{Histogram of poisoning effort ratio in the toy experiment}
\label{Attack:cost}
\end{figure}

We now analyze a single, representative attack trial to gain deeper insight into the attack strategy.
In this trial, the UCBs of the $5$ arms without attack are
\[
\text{pre-attack: } (0.204, 0.097, 0.959, 0.507, 0.818)\,.
\]
That is, arm 3 would have been chosen.
As mentioned earlier, $a^*=2$ is chosen to be the target arm 
as it has the smallest pre-attack UCB. 
After attack, the UCBs of all arms become:
$$\text{post-attack: } (0.204,0.605,0.604,0.507,0.604).$$
The attacker successfully forced the bandit to choose arm 2.
It did so by poisoning the historical data to make arm 2 look better and arms 3 and 5 look worse.
It left arms 1 and 4 unchanged.

Figure~\ref{Attack:reward}  shows the attack where each panel is the historical rewards where that arm was chosen.  We show the original rewards ($y_{ai}$, blue circle) and post-attack rewards ($y_{ai}+\Delta_{ai}$, red cross) for all historical points $i$ where arm $a$ was chosen.
Intuitively, to decrease the UCB of arm $a$ the attacker should reduce the reward if the historical context $x$ is ``similar'' to $x^*$, and boost the reward otherwise. 
To see this, we sort the historical points by the inner product $x^\top x^*$ in ascending order. 
As shown in Figure~\ref{Attack:reward}\subref{reward3} and~\subref{reward5}, the attacker gave the illusion that these arms are not good for $x^*$ by reducing the rewards when $x^\top x^*$ is large. 
The attacker also increased the rewards when $x^\top x^*$ is very negative, which reinforces the illusion.
In contrast, the attacker did the opposite on the target arm as shown in Figure~\ref{Attack:reward}\subref{reward2}.

\begin{figure}[H]
\centering
\subfloat[arm 1]{
\includegraphics[width=0.18\textwidth,height=0.16\textwidth]{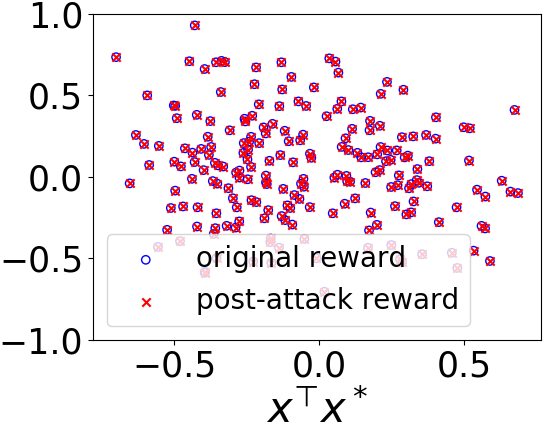}
\label{reward1}
}
\subfloat[arm 2]{
\label{reward2}
\includegraphics[width=0.18\textwidth,height=0.16\textwidth]{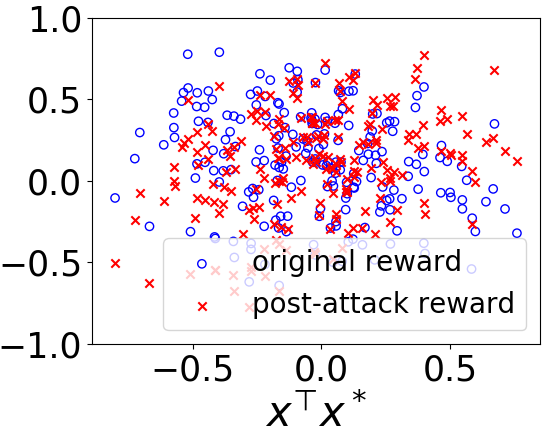}
}
\subfloat[arm 3]{
\includegraphics[width=0.18\textwidth,height=0.16\textwidth]{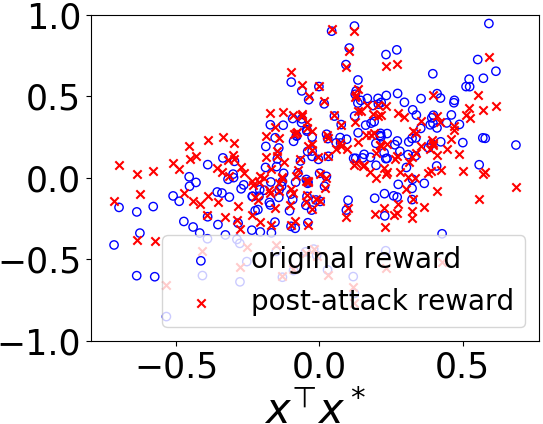}
\label{reward3}
}
\subfloat[arm 4]{
\includegraphics[width=0.18\textwidth,height=0.16\textwidth]{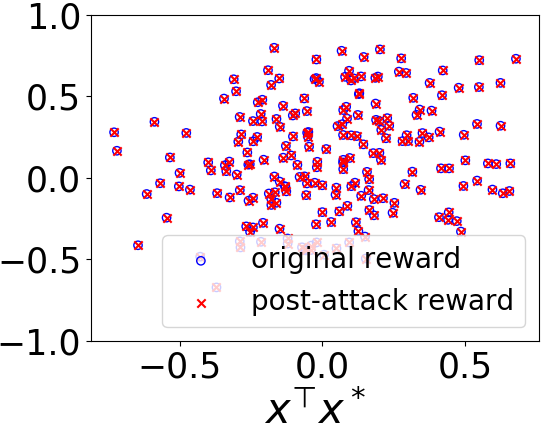}
\label{reward4}
}
\subfloat[arm 5]{
\includegraphics[width=0.18\textwidth,height=0.16\textwidth]{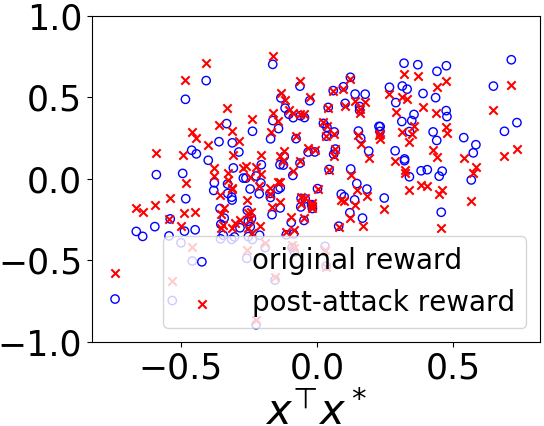}
\label{reward5}
}
\caption{Original reward $y_{ai}$ and post-attack reward $y_{ai}+\Delta_{ai}$ for each arm.}
\label{Attack:reward}
\end{figure}

\begin{figure}[H]
\centering
\subfloat[arm 1]{
\includegraphics[width=0.18\textwidth,height=0.16\textwidth]{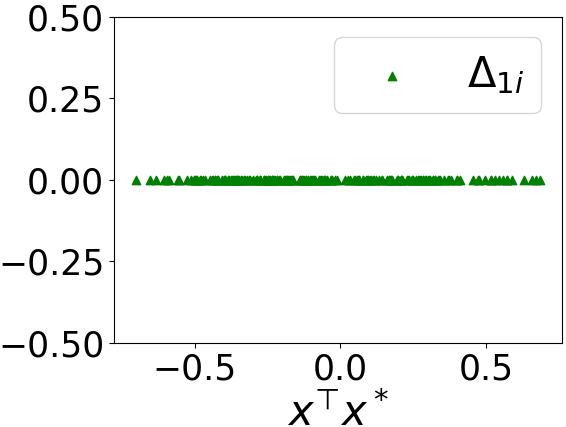}
\label{reward_diff1}
}
\subfloat[arm 2]{
\label{reward_diff2}
\includegraphics[width=0.18\textwidth,height=0.16\textwidth]{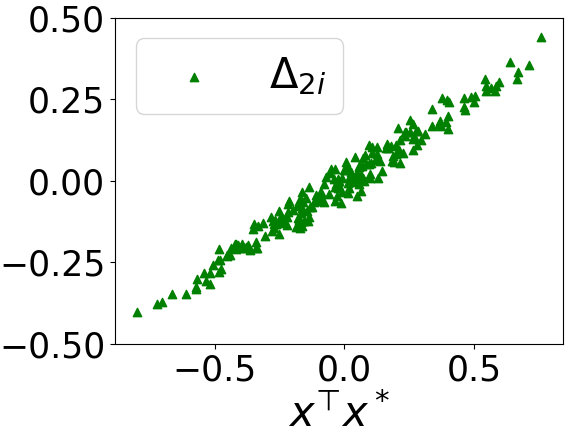}
}
\subfloat[arm 3]{
\includegraphics[width=0.18\textwidth,height=0.16\textwidth]{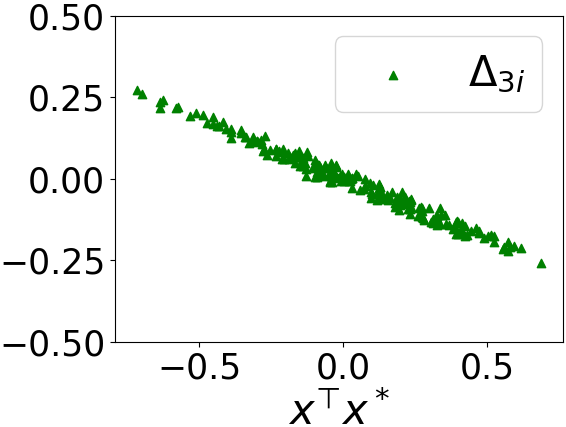}
\label{reward_diff3}
}
\subfloat[arm 4]{
\includegraphics[width=0.18\textwidth,height=0.16\textwidth]{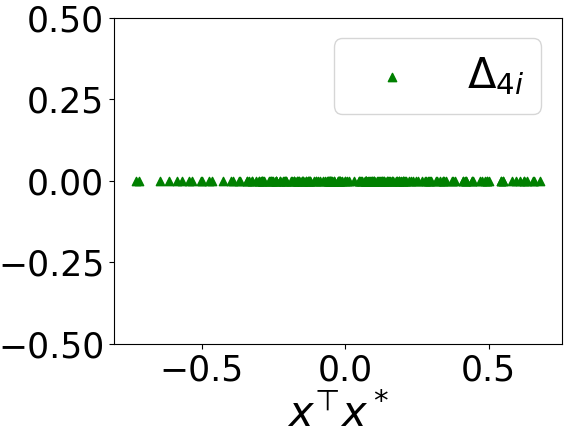}
\label{reward_diff4}
}
\subfloat[arm 5]{
\includegraphics[width=0.18\textwidth,height=0.16\textwidth]{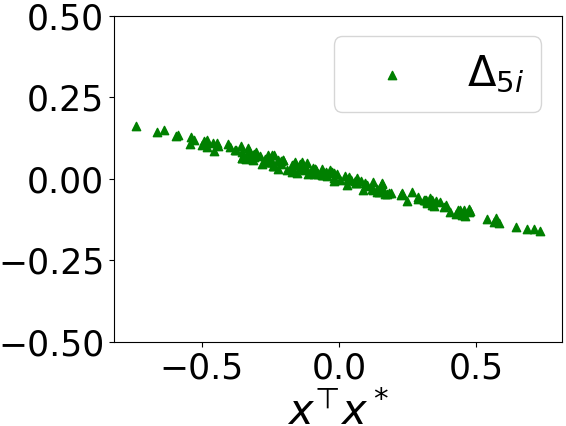}
\label{reward_diff5}
}
\caption{The reward poisoning $\Delta_{ai}$ for each arm.}
\label{Attack:reward_diff}
\end{figure}

\subsection{Attack on Real Data: Yahoo! News Recommendation}

To further demonstrate the effectiveness of the attack algorithm in real applications, we now test it on the Yahoo! Front Page Today Module User Click Log Dataset (R6A).\footnote{URL: \url{https://webscope.sandbox.yahoo.com/catalog.php?datatype=r}~.}
The dataset contains a fraction of user click log for news articles
displayed in the Featured Tab of the Today Module on Yahoo! Front Page (http://www.yahoo.com) during the first ten days in May 2009. Specifically, it contains about $46$ million user visits, 
where each user is represented as a $6$-dimensional contextual vector. 
When a user arrives, the Yahoo! Webscope program selects an article (an arm) from a candidate article pool and displays it to the user. 
The system receives reward $1$ if the user clicks on the article and $0$ otherwise.  Contextual information about users can be found in prior work~\cite{li10contextual}.

To apply the attack algorithm, we require that the set of arms remain unchanged. 
However, the Yahoo! candidate article pool (i.e., the set of arms) varies as new articles are added and old ones are removed over time.
Nonetheless, there are long periods of time where the set of arms is fixed.
We restrict ourselves to such a stable time period for our experiment (specifically the period from 7:25 to 10:35 on May 1, 2009)
in the Yahoo! data, which contains $243$,$667$ user visits.
During this period the bandit has $K=20$ fixed arms. 
We further split the time period such that the first $n=8000$ user visits are used as the historical training data to be poisoned, and the remaining $m=163,667$ data points as the test data. 
The bandit learning algorithm uses regularization $\lambda=1$. The confidence parameter is $\delta=0.05$. 
The subGaussian parameter is set to $\sigma=\frac{1}{4}$ for binary rewards.

We simulate attacks on three target user context vectors: The most frequent user context vector $x^*=\bar x$, a middle user context vector $x^*=$\sout{$x$}, 
and the least frequent user context vector $x^*=\underline x$ in the test data. These three user context vectors appeared $5508$, $106$, and $1$ times, respectively, in the test data. 
Note that there are potentially many distinct real-world users that are mapped to the same user contextual vector, therefore the ``user'' in our experiment does not necessarily mean a real-world individual that appeared thousands of times.

We again choose as the target arm $a^*$ the worst arm on the target user as defined by~\eqref{eq:worstUCB}. To determine the target arm, we first simulate the bandit algorithm on the original (pre-attack) training data, and then pick the arm with the smallest UCB for that user. For the three target users we consider, the target arms are $8$, $3$, and $8$ respectively.  The attacker uses attack margin $\epsilon=0.001$. 

Different from the toy example where the reward can be any value in $\R$, the reward in the Yahoo! dataset must be binary, corresponding to a click-or-not outcome of the recommendation. Therefore, the attacker must enforce $y_{ai}+\Delta_{ai}\in\{0,1\}$. However, this results in a combinatorial problem.   To preserve convexity, we instead relax the attacked reward into a box constraint: $y_{ai}+\Delta_{ai}\in[0,1]$. We add these new constraints to~\eqref{attack:optimization_strong} and solve the following optimization: 
\begin{equation}\label{attack:optimization_real}
\begin{aligned}
\min_{\Delta \in \R^n}\quad &\sum_{a\in [K]}\Vert\Delta_a\Vert_2^2\\
\text{s.t. \quad} &{x^*}^\top\hat\theta_{a^*}+\alpha_{a^*}\Vert x^* \Vert_{V_{a^*}^{-1}}
	\ge\epsilon+{x^*}^\top\hat\theta_a+\alpha_a\Vert x^* \Vert_{V_a^{-1}}, ~~\forall a\neq a^*,\\
	& y_{ai} + \Delta_{ai}\in [0, 1], ~~\forall i\in[m_a], ~~\forall a,\\
\text{where \quad}&\hat\theta_a=V_a^{-1}X_a^\top(y_a+\Delta_a), ~~\forall a.
\end{aligned}
\end{equation}

After the real-valued $\Delta_{ai}$ is computed, the attacker performs rounding to turn $y_{ai}+\Delta_{ai}$ into $0$ or $1$.  Specifically, the attacker thresholds $y_{ai}+\Delta_{ai}$ with a constant $c\in[0,1]$, so that if $y_{ai}+\Delta_{ai}>c$, then let the post-attack reward be $1$, otherwise let the post-attack reward be $0$. 
Note that the poisoned rewards now correspond to ``reward flipping'' from $0$ to $1$ or vice versa by the attacker. 
In our experiment, we let the attacker try out $10^4$ thresholds $c$ equally distributed in $[0,1]$.
The attacker examines different thresholds for two concerns. First, there is no guarantee that the thresholded solution still triggers the target arm pull, thus the attacker needs to check if the selected arm for $x^*$ is $a^*$. If not, the corresponding threshold $c$ is inadmissible. Second, among those thresholds that indeed trigger the target arm pull, the attacker selects the one that minimizes the number of flipped rewards, which corresponds to the smallest poisoning effort in the binary reward case. 

In Table~\ref{table:real_experiment}, we summarize the experimental results for attacking the three target users. Note that the attack is successful on all three target users. The best thresholds $c$ for $\bar x$, \sout{$x$} and $\underline x$ are $0.0449$, $0.1911$, and $0.0439$, respectively.
The number of flipped rewards is small compared to $n=8000$, which demonstrates that the attacker only needs to spend little cost in order to force the bandit to pull the target arm. Note that the poisoning effect ratio is relatively large. This is because most of the pre-attack rewards are 0, in which case the denominator in~\eqref{def:poisoning_effect_ratio} is small.

\begin{table*}[ht]
	\small
	\centering
	\begin{tabular}{ |p{56mm}|p{2cm}|p{2cm}|p{2cm}|} 
		\hline
		& $\bar x$ & \sout{$x$} & $\underline x$\\
		\hline
		strong attack successful? & True & True & True \\
		\hline
		number [percentage] of flipped rewards & $82$ [$1.0\%$]  & $9$ [$0.1\%$] & $19$ [$0.2\%$] \\
		\hline
		poisoning effort ratio& 0.572 & 0.189 & 0.275\\
		\hline
	\end{tabular}
	\caption{Results of experiments on Yahoo! data}\label{table:real_experiment}
\end{table*}
In Figure~\ref{Attack:realdata}, we show the reward poisoning $\Delta$ on the historical data against the three target users, respectively. In all three cases, only a few rewards of the target arm are flipped from $0$ to $1$ by the attacker while those of the other arms remain unchanged.  Therefore, we only show the reward poisoning on historical data restricted to the target arm (namely on $y_{a^*}$). The $82$ and $19$ flipped rewards overlap in Fig.~\ref{Attack:realdata}~\subref{reward_diff_mfu} and Fig.~\ref{Attack:realdata}~\subref{reward_diff_lfu}. Note that the contexts of those flipped rewards are highly correlated with $x^*$.
\begin{figure}[H]
\centering
\subfloat[Most frequent user $x^*=\bar x$]{
\includegraphics[width=0.33\textwidth,height=0.25\textwidth]{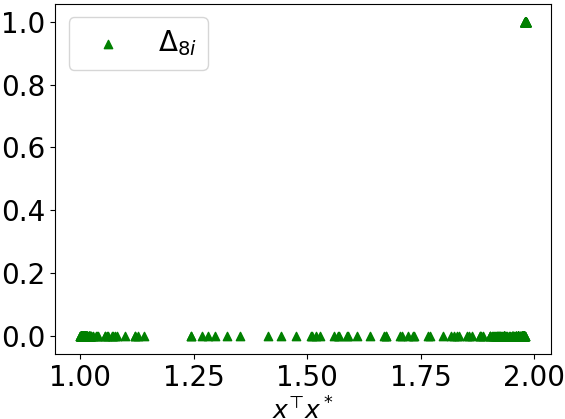}
\label{reward_diff_mfu}
}
\subfloat[Medium frequent user $x^*=$\sout{$x$}]{
\label{reward_diff_medu}
\includegraphics[width=0.33\textwidth,height=0.25\textwidth]{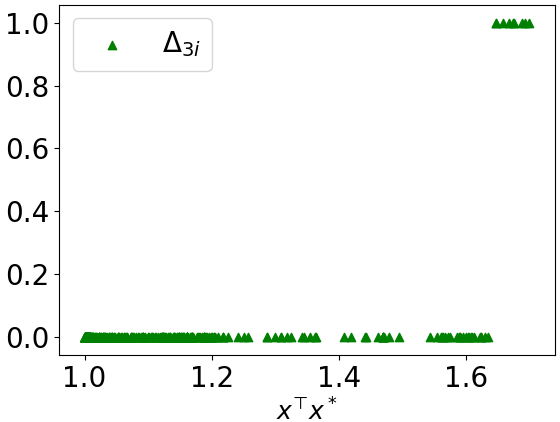}
}
\subfloat[Least frequent user $x^*=\underline x$]{
\includegraphics[width=0.33\textwidth,height=0.25\textwidth]{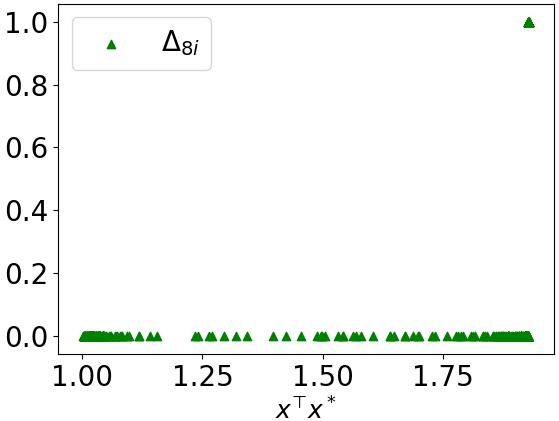}
\label{reward_diff_lfu}
}
\caption{The reward poisoning $\Delta_{ai}$ on three target users.}
\label{Attack:realdata}
\end{figure}

\subsection{Study on Feasibility}
\label{subsec:feasibility}
The attack feasibility depends on the historical contexts $X$, the bandit algorithm-specific UCB parameter $\alpha$, the attack margin $\epsilon$, the target arm $a^*$, and the target context $x^*$. 
To visualize the infeasible region of strong attack on context, we consider the following toy example. 

The bandit has $K=4$ arms.  
The attacker's target arm is $a^*=4$, and the target context $x^*$ lies in $\R^3$.
The historical context vectors are 
\begin{equation}\label{data:hist}
X_1=[1,~0,~0], ~~ X_2=[0,~-1,~1], ~~ X_3=[0, ~2, ~0], ~~ X_4=[2, ~0, ~0].
\end{equation}
The problem parameters are $\sigma=S=\lambda=\epsilon=1$ and $\delta=0.05$. 
According to Theorem~\ref{thm:infeasible}, any infeasible target context $x^*$ satisfies $X_4 x^*=0$. 
Thus such $x^*$ must lie in the subspace spanned by the $y$-axis and $z$-axis. 
This allows us to show infeasible regions as 2D plots.
In Figure~\ref{InfeasibleRegions:addcontext}\subref{infeasible:1_1}, we show the infeasible regions.
We distinguish the infeasible region due to each non-target arm by a different color. 
For example, the infeasible region due to arm 1 consists of all contexts on which the target arm $a^*$ can never be $\epsilon$-better than arm 1 regardless of the attack. 
Note that the infeasible region due to arm 2 is a line segment of finite length, while that due to arm 3 is the whole $y=0$ line. The shape of the infeasible region due to each non-target arm varies because the historical data differs and therefore the conditions in theorem~\ref{thm:infeasible} characterizes different shapes.
Note that the origin $x=0$ satisfies the conditions in Theorem~\ref{thm:infeasible} and therefore is always infeasible. 

One important observation is that, if the bandit algorithm is trained on more historical data,  more context vectors $x^*$ can potentially be strongly attacked.
Formally,
as indicated by Theorem~\ref{thm:infeasible} 
as the null space of historical context matrices $X_a, a\in [K]$ shrinks, the infeasible region shrinks as well.
To demonstrate this,
in Figure~\ref{InfeasibleRegions:addcontext}\subref{infeasible:1_2} we add a context [0, 0, 0.5] to $X_1$ such that the historical contexts are:
\begin{equation}
\label{eq:historicalX}
X_1=\begin{bmatrix}1,~0,~0\phantom{.0}\\0,~0,~0.5\end{bmatrix}, ~~X_2=[0,~-1,~1], ~~X_3=[0,~2,~0], ~~X_4=[2,~0,~0]\,.
\end{equation}
Now that $\Null(X_1)$ is reduced, the infeasibility region due to arm 1 shrinks from the circle in Figure~\ref{InfeasibleRegions:addcontext}\subref{infeasible:1_1} to a horizontal line segment in Figure~\ref{InfeasibleRegions:addcontext}\subref{infeasible:1_2}. However the infeasible region may not shrink to a subset of itself, as indicated by the line segment having wider length along $y$ axis than the original circle, thus the shrink happens in the sense of being restricted to a lower-dimensional subspace.

Next we add a historical context $[0, 1, 0]$ to $X_4$: 
$$X_1=\begin{bmatrix}1,~0,~0\phantom{.0}\\0,~0,~0.5\end{bmatrix}, ~~X_2=[0,~-1,~1], ~~X_3=[0,~2,~0], ~~X_4=\begin{bmatrix}2,~0,~0\\0,~1,~0\end{bmatrix}\,.$$
Then the infeasibility region due to arm 1 and arm 2 both shrink to the origin while arm 3 becomes a line segment, as shown in Figure~\ref{InfeasibleRegions:addcontext}\subref{infeasible:1_3}. 
\begin{figure}[H]
\subfloat[original data]{
\includegraphics[width=0.33\textwidth,height=0.28\textwidth]{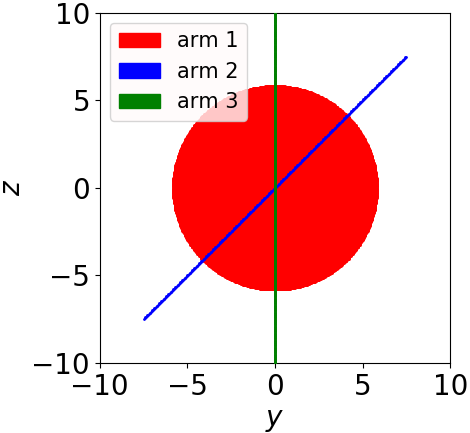}
\label{infeasible:1_1}
}
\subfloat[Context added to $X_1$]{
\label{infeasible:1_2}
\includegraphics[width=0.33\textwidth,height=0.28\textwidth]{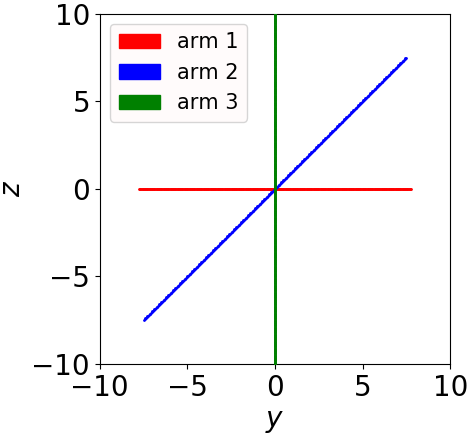}
}
\subfloat[Context added to $X_4$]{
\includegraphics[width=0.33\textwidth,height=0.28\textwidth]{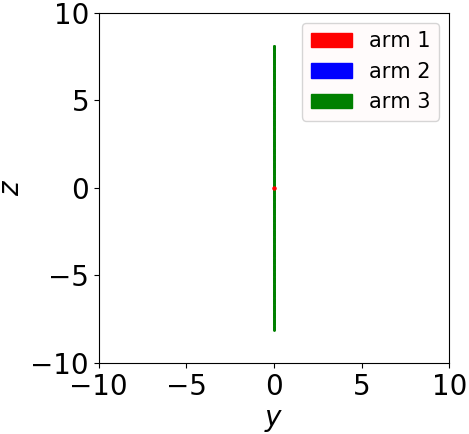}
\label{infeasible:1_3}
}
\caption{Infeasible region due to each non-target arm.}
\label{InfeasibleRegions:addcontext}
\end{figure}
In practice, historical data is often abundant so that $\forall a\neq a^*$, $ X_{a^*}\cup  X_a$ spans the whole $\mathcal{R}^d$ space, and the only infeasible point is the origin.
That is, the attacker can choose to attack essentially any context vector.

Another observation is that the infeasible region shrinks as the attack margin $\epsilon$ decreases, as shown in Figure~\ref{InfeasibleRegions:eps}. The historical data for each arm is the same as~\eqref{data:hist}.
The reason is that a smaller $\epsilon$ makes the constraints in~\eqref{attack:optimization_strong} easier to satisfy and therefore more contexts are feasible. 
As $\epsilon\rightarrow0$ the infeasible region converges to those contexts that cannot be weakly attacked, which in this example is the line $y=0$ in Figure~\ref{InfeasibleRegions:eps}\subref{infeasible:2_3}. Note that the contexts that cannot be weakly attacked are those that make~\eqref{attack:optimization_weak} infeasible. Therefore, we see that without abundant historical data, there will be some contexts that can never be strongly attacked even when $\epsilon\rightarrow 0$.
Also note that the origin $x^*=0$ can never be strongly attacked by definition.

\begin{figure}[H]
\subfloat[$\epsilon=1$]{
\includegraphics[width=0.33\textwidth,height=0.28\textwidth]{infeasible1.png}
\label{infeasible:2_1}
}
\subfloat[$\epsilon=0.5$]{
\label{infeasible:2_2}
\includegraphics[width=0.33\textwidth,height=0.28\textwidth]{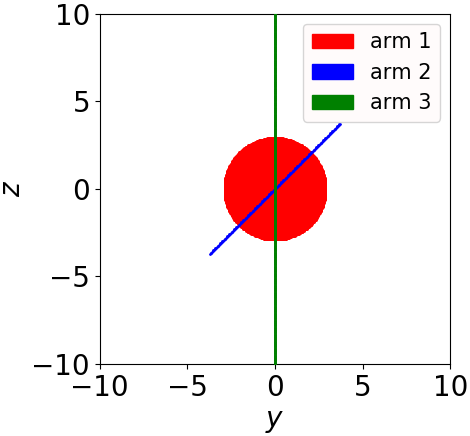}
}
\subfloat[$\epsilon=0.1$]{
\includegraphics[width=0.33\textwidth,height=0.28\textwidth]{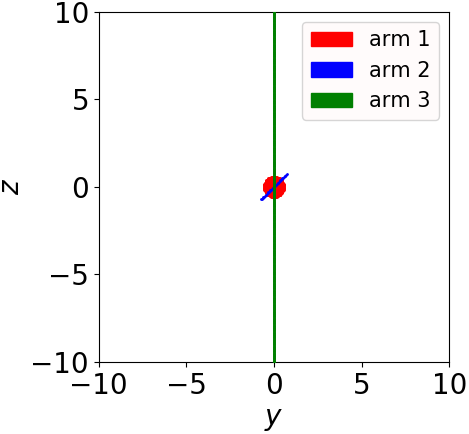}
\label{infeasible:2_3}
}
\caption{Infeasible region shrinks as attack margin $\epsilon$ decreases.}
\label{InfeasibleRegions:eps}
\end{figure}

\subsection{Study on Side Effects}

We first give an intuitive illustration of the side effect in 2D space. The bandit has $K=3$ arms, where the arm parameters are $ \theta_a$. We generate $n=1000$ historical data same as before with noise $\sigma=0.1$. The target context $x^*$ is uniformly sampled from $\X$. The bandit algorithm uses regularization weight $\lambda=1$ and confidence parameter $\delta=0.05$. Without attack, the UCB for the three arms are
\begin{equation}
\text{pre-attack: }(-0.419, ~0.192,~1.013).
\end{equation}
Therefore without attack arm 3 would have been chosen. By our design choice, the target arm is $a^*=1$. The attacker uses margin $\epsilon=0.001$. After attack the UCBs of all arms become:
\begin{equation}
\text{post-attack: }(0.290, ~0.192,~0.289).
\end{equation}
As shown in Figure~\ref{side_effect_2D}, the attacker forces the post-attack parameter of the best arm $\hat\theta_3$ to deviate from $x^*$ while making $\hat\theta_1$ closer to $x^*$. Note that the attacker could also change the norm of the parameter. Note that arm 2 is not attacked, thus $\theta_2$ and $\hat\theta_2$ overlap. The side effect is denoted by the brown arcs on the circle, where the arms chosen for those contexts are changed by the attacker. The side effect fraction for this example is $\hat s=0.315$.

\begin{figure}[H]
\centering
\includegraphics[width=0.5\textwidth,height=0.45\textwidth]{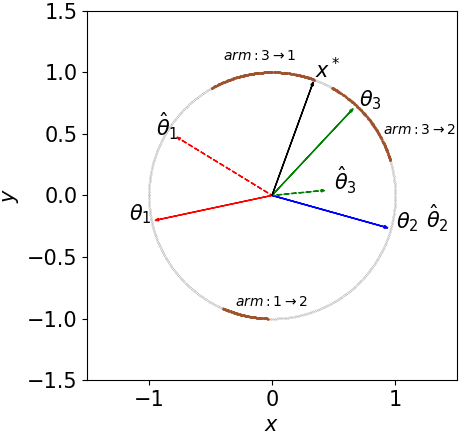}
\caption{Side effect shown in 2D context space.}
\label{side_effect_2D}
\end{figure}

Now we design a toy experiment to study how the side effect depends on the number of arms and the problem dimension. The context space $\X$ 
is the $d$-dimensional sphere $\Sd$ and $P$ is uniform on the sphere. The bandit has $K$ arms, where the arm parameters are sampled from $P$. Same as before, we generate $n=2000$ historical data with noise $\sigma=0.1$. The bandit algorithm uses regularization weight $\lambda=1$. The target context $x^*$ is sampled from $P$.  The attacker's margin is $\epsilon=0.001$ and the target arm $a^*$ is the worst arm on the target context $x^*$. We sample $m=10^3$ contexts from $P$ to evaluate $\hat s$.

In Figure~\ref{side_effect_arm}, we fix $d=2$ and show a histogram of $\hat s$ as the number of arm varies. Note that the attack affects about $30\%$ users. The median $\hat s$ for the three panels are $0.249$, $0.317$, and $0.224$ respectively, which shows that the side effect does not grow with the number of arms.
\begin{figure}[H]
\subfloat[$K=2$]{
\includegraphics[width=0.33\textwidth,height=0.28\textwidth]{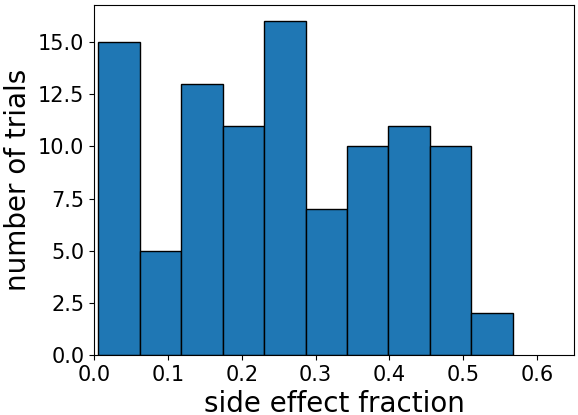}
\label{se_a5}
}
\subfloat[$K=20$]{
\label{se_a20}
\includegraphics[width=0.33\textwidth,height=0.28\textwidth]{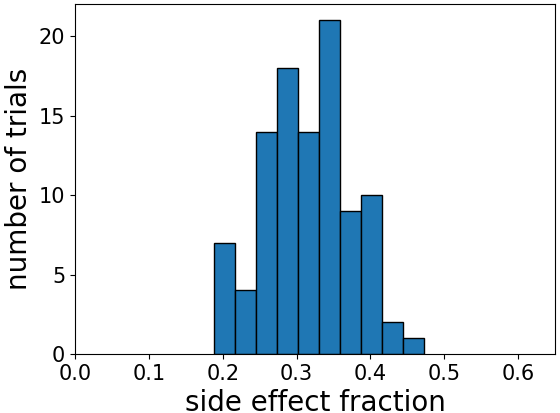}
}
\subfloat[$K=200$]{
\includegraphics[width=0.33\textwidth,height=0.28\textwidth]{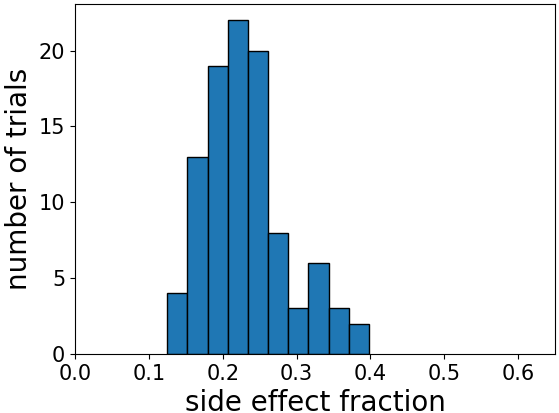}
\label{se_a50}
}
\caption{side effect fraction as arm number $K$ increases.}
\label{side_effect_arm}
\end{figure}

In Figure~\ref{side_effect_d}, we fix $K=5$ and show the  side effect as the dimension $d$ varies. The median $\hat s$ for the three panels are $0.435$, $0.090$, and $0.035$, respectively, which implies that in higher dimensional space, the side effect tends to be smaller.
\begin{figure}[H]
\subfloat[$d=2$]{
\includegraphics[width=0.33\textwidth,height=0.28\textwidth]{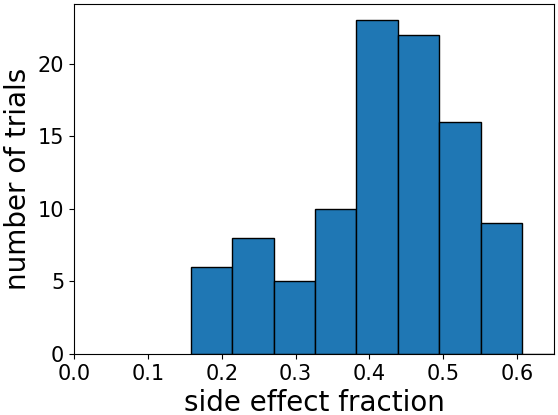}
\label{se_d5}
}
\subfloat[$d=20$]{
\label{se_d20}
\includegraphics[width=0.33\textwidth,height=0.28\textwidth]{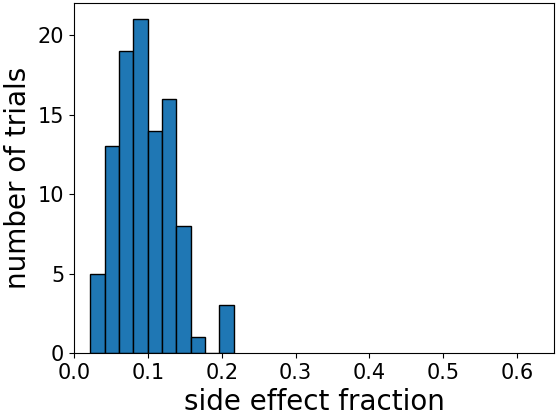}
}
\subfloat[$d=200$]{
\includegraphics[width=0.33\textwidth,height=0.28\textwidth]{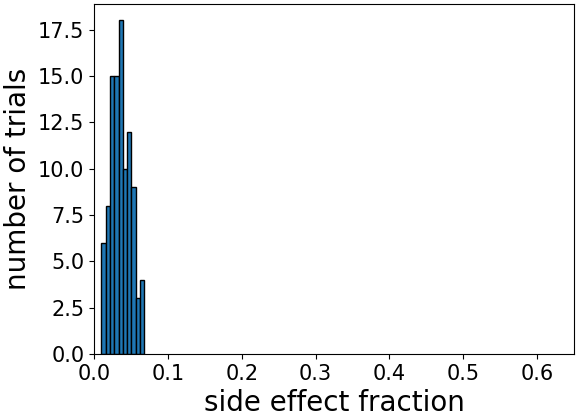}
\label{se_d50}
}
\caption{side effect fraction as dimension $d$ increases.}
\label{side_effect_d}
\end{figure}
As the dimension $d$ increases, the attack has less side effect. This exposes the hazard that in real-world applications where the problem dimension is high, the attack will be hard to detect from side effects. 

We also study the side effect for the real data experiment. There we use the $m=163,667$ test users to evaluate the side effect. The side effect fraction for the three users are $0.5391$, $0.0750$, and $0.5040$, respectively. Note that the most frequent user and the least frequent user have a large side effect, which makes the attack easy to detect. In contrast, the side effect of the medium frequent user is extremely small. This implies that the attack can induce different level of side effect for different target users.

\section{Conclusions and Future Work}
We studied offline data poisoning attack of contextual bandits. 
We proposed an optimization-based attack framework against contextual bandit algorithms. By manipulating the historical rewards, the attack can successfully force the bandit algorithm to pull a pre-specified arm for some target context. Experiments on both synthetic and real-world data demonstrate the effectiveness of the attack.
This exposes a security concern in AI systems that involve contextual bandits.

There are several future directions that can be explored. For example, our current attack only targets a single context $x^*$. Future work can characterize how to target a set of contexts simultaneously, i.e., force the bandit algorithm to pull the target arm for all contexts in some target set. In the simplest case where the set contains finitely many contexts, one can just replicate the constraint in~\eqref{attack:optimization_strong} for each context in the set. The situation is more complicated if the target set is infinite or just too large. 
Another interesting question is how to develop defense mechanisms to protect the bandit from being attacked. As indicated in this paper, the defender can rely on the side effect to sense the existence of attacks.  
Conversely, it is also an open question how the attacker might attempt to minimize its side effect during the attack, so that the chances of being detected are minimized.
Finally, in this paper we restrict the ability of the attacker to manipulating only the historical rewards. However, there are other types of attacks such as poisoning the historical contexts, adding additional data points, removing existing data points, or combinations of the above. The problem could become non-convex or even combinatorial depending on the type of the attack; some of these settings have been studied under the name ``machine teaching''~\cite{Zhu2015Machine,Zhu2018Overview}. Future work needs to identify how to extend our current attack framework to more general settings.

\textbf{Acknowledgment}
This work is supported in part by NSF
1545481, 
1704117,
1623605,
1561512,
and the MADLab AF Center of Excellence FA9550-18-1-0166.

\bibliographystyle{splncs04}
\bibliography{ref}

\end{document}